\documentclass[letterpaper]{article} 
\usepackage{aaai25}  
\usepackage{times}  
\usepackage{helvet}  
\usepackage{courier}  
\usepackage[hyphens]{url}  
\usepackage{graphicx} 
\usepackage{natbib}  
\usepackage{caption} 
\usepackage{algorithm}
\usepackage{algorithmic}

\usepackage{amssymb}
\usepackage{amsmath}
\usepackage{subcaption}
\usepackage{booktabs}
\usepackage{float}
\usepackage{multirow}
\usepackage{amsthm}

\usepackage{xcolor}

\newcommand{\bybcomment}[1]{}
\newcommand{\stcomment}[1]{}
\newcommand{\Omit}[1]{}

\newif\ifsmalltable\smalltablefalse

\newcommand{\pog}{TILG}

\newcommand{\capsfullpog}{Typed Instance Learning Graph}
\newcommand{\fullpog}{typed instance learning graph}
\newcommand{\iso}{\tau}
\newcommand{\ncol}{c}
\newcommand{\ecol}{l}

\newtheorem{theorem}{Theorem}
\newtheorem{definition}{Definition}
\newtheorem{lemma}{Lemma}

\usepackage{newfloat}
\usepackage{listings}
\DeclareCaptionStyle{ruled}{labelfont=normalfont,labelsep=colon,strut=off} 
\floatstyle{ruled}
\newfloat{listing}{tb}{lst}{}
\floatname{listing}{Listing}

\title{Learning Efficiency Meets Symmetry Breaking}
\author {
    Yingbin Bai\textsuperscript{\rm 1},
    Sylvie Thi\'ebaux\textsuperscript{\rm 1, 2},
    Felipe Trevizan\textsuperscript{\rm 1}
}

\affiliations {
    \textsuperscript{\rm 1}School of Computing, The Australian National University\\
    \textsuperscript{\rm 2}LAAS-CNRS, Universit\'e de Toulouse\\
    yingbin.bai@anu.edu.au, sylvie.thiebaux@anu.edu.au, felipe.trevizan@anu.edu.au
}

\begin{document}

\maketitle

\begin{abstract}
Learning-based planners leveraging Graph Neural Networks can learn search guidance applicable to large search spaces, yet their potential to address symmetries remains largely unexplored.
In this paper, we introduce a graph representation of planning problems allying learning efficiency with the ability to detect symmetries, along with two pruning methods, action pruning and state pruning, designed to manage symmetries during search. The integration of these techniques into Fast Downward achieves a first-time success over LAMA on the latest IPC learning track dataset. Code is released at: \url{https://github.com/bybeye/Distincter}.
\end{abstract}

\section{Introduction}
Over the past two decades, heuristic search has achieved significant success across a variety of planning problems, and has become the standard approach in the field \citep{richter:westphal:10, bercher:etal:20, correa:etal:22, geisser:etal:22, klossner:etal:23}. Nevertheless, even in classical planning, scalability remains a significant challenge for these methods. This has led a growing number of researchers to turn to learning-based methods, particularly using Graph Neural Networks (GNNs) \citep{toyer:etal:20, shen:etal:20, stahlberg:etal:22, chen:etal:24:icaps, horcik:sir:24, hao:etal:24, drexler:etal:24:kr}. Unlike traditional model-based methods, which are reliant solely on analysing planning domain and problem definitions, GNNs are capable of learning patterns and strategies from existing plans to enhance search efficiency and adaptability.

However, learning efficiency alone is insufficient to address the challenges inherent in large-scale planning, which often involves a substantial number of symmetrical states~\cite{wehrle:etal:15, sievers:etal:19:icaps}.
Although these states do not affect plan quality, they consume significant computational resources and can considerably slow down the search process.
In this paper, we use NNs and GNNs with permutation invariant activation functions to learn a permutation invariant function allowing them to produce consistent outputs for symmetrical inputs.
Despite this advantage, the full potential of this feature has not yet been effectively harnessed to detect and break symmetries during the search process.

In this paper, we remedy this by introducing  a graph representation
designed to achieve two key objectives: learning efficiency and symmetry reduction. Leveraging the strengths of this representation, we propose two pruning methodologies: action pruning and state pruning. Action pruning infers symmetries by analyzing object involvement in action parameters, without generating child states nor computing their heuristic value. Additionally, since GNNs can retain invariant outputs for symmetrical inputs, state pruning exploits this property to efficiently identify symmetries between states.

To evaluate the proposed techniques, we implemented them on top of Fast Downward \cite{helmert:06} in a planner called Distincter and carried out experiments on the 2023 International Planning Competition Learning Track. The overall coverage of Distincter surpasses that of the traditional SOTA method, LAMA \cite{richter:westphal:10}, for the first time in the recent literature on learning planning heuristics, marking a significant milestone for learning-based methods.

In terms of related work, recent independent work by \cite{drexler:etal:24:kr} removes symmetries in the training set in offline mode, thereby improving training effectiveness. In contrast, our approach focuses on removing symmetries during the search process, so as to enhance search efficiency and scale to large planning problems.

\section{Background and notation}
A lifted planning problem is defined as a tuple $\Pi = \langle \mathcal{O}, \mathcal{T}, \mathcal{P}, \mathcal{A}, \mathcal{I}, \mathcal{G}\rangle$, where $\mathcal{O}$ denotes a set of objects,
$\mathcal{T}$ is a set of object types, $\mathcal{P}$ consists of first-order predicates, $\mathcal{A}$ comprises action schemas, $\mathcal{I}$ specifies the current (or initial) state, and $\mathcal{G}$ delineates the goal.

A predicate $p \in \mathcal{P}$ has parameters $x_{p_1}, \ldots, x_{p_n}$ for $p_n \in \mathbb{N}$, where each parameter requires a specific type of object. A predicate can be instantiated by assigning each $x_i$ to an object from $\mathcal{O}$, resulting in a proposition $\rho$.
A state is an assignment of truth value to the propositions.

An action schema $a = \langle X_a, \textit{pre}(a), \textit{add}(a), \textit{del}(a) \rangle $ is defined as a tuple comprising a list of typed parameters $X_a = (x_{a_1}, \ldots x_{a_n})$, along with sets of preconditions, add effects, and delete effects, all of which are predicates in $\mathcal{P}$ with parameters from $X_a$. When all parameters of an action schema are instantiated with objects of the required types, the action is referred to as a ground action. A ground action $a$ is applicable in a state $s$ if $\textit{pre}(a) \subseteq s$. When $a$ is applied to $s$, the resulting state $s'$ is given by $(s \setminus \textit{del}(a)) \cup \textit{add}(a)$. In this context, the state $s$ is referred to as the parent state, and $s'$ is known as the child state. Since the set of applicable actions for a parent state is typically not a singleton, expanding a parent state usually generates a set of child states.

A sequence of actions $a_1, \ldots, a_n$ is applicable in a state $s$ if there exists a sequence of states $s_0, \ldots, s_n$ such that $s_0 = s$, and for each $i \in \{1, \ldots,n\}$, the state $s_i$ is the result of applying $a_i$ in $s_{i-1}$. The aim is to find a plan for a given planning problem $\Pi$, which is a sequence of ground actions that is applicable in the initial state $\mathcal{I}$ and results in a state $s_n$ such that $\mathcal{G} \subseteq s_n$.

A colored (or labelled) graph is a tuple $G = \langle V, E, \ncol, \ecol \rangle$ where $V$ is the set of vertices, $E$ is the set of edges, and $\ncol$ (resp.\ $\ecol$) maps vertices (resp.\ edges) to their color.
Two graphs $G = \langle V, E,\ncol, \ecol \rangle$ and $G' = \langle V', E',\ncol', \ecol' \rangle$ are isomorphic, denoted by $G \cong G'$, if there exists a bijection $\iso: V \rightarrow V'$ such that $(u, v) \in E$ iff $(\iso(u), \iso(v)) \in E'$, $\ncol'(\iso(v))=\ncol(v)$ for all $v\in V$, and $\ecol'((\iso(u),\iso(v))) = \ecol((u,v))$ for all $(u,v)\in E$.

An automorphism of $G$ is defined as an equivalence relation $\sigma$ representing an isomorphism between $G$ and itself.
The set of all automorphisms of $G$ forms a group under the operation of composition, known as the automorphism group $\text{Aut}(G)$ of the graph.
The orbit of a vertex $v$ in a graph consists of all vertices that can be transformed into $v$ by any automorphism in $\text{Aut}(G)$.
This implies that any two vertices within the same orbit are structurally equivalent in the graph, maintaining the same connections and roles relative to other vertices and edges.

\section{Distincter}

\subsection{\capsfullpog\ (\pog)}

Our graph representation extends the Instance Learning Graph (ILG) \cite{chen:etal:24:icaps}, maintaining similar structures but offering additional information for learning and symmetry detection.
The graph's
vertices represent objects and propositions in the initial (current) state and the goal, and edges exist between propositions and the objects in their parameter list.\footnote{In the following, we will use the word \textit{symmetric} to refer to states represented by isomorphic TILG and to objects or propositions that are related via $\iso$ (or $\sigma$ depending on the context).}
Vertex features capture the object types, the predicates instanciated by the propositions, and whether goal propositions have been achieved. Edge features capture the index of objects in proposition parameter lists. Formally:

\begin{definition}
  Let $\Pi = \langle \mathcal{O}, \mathcal{T}, \mathcal{P}, \mathcal{A}, \mathcal{I}, \mathcal{G} \rangle$ represent a lifted planning problem.
  The \fullpog\ (\pog) for $\Pi$ is the undirected graph $G_{\Pi} = \langle V, E, f, l \rangle$, such that:
\begin{itemize}
    \item $V = \mathcal{O} \cup \mathcal{I} \cup \mathcal{G}$
    \item $E = \{ (o, p(o_1,..., o_n)) | o \in \mathcal{O}, \exists i\: o = o_i, p(o_1,..., o_n) \in \mathcal{I} \cup \mathcal{G}\}$
    \item $\ncol\!: \!V \!\rightarrow \!\{(status, class) \! \mid \!status \!\in \!\{0, 1, 2, 3\}, class \in \!\mathcal{T} \!\cup\! \mathcal{P} \}$, maps each vertex to a tuple where: 
        \begin{itemize}
        \item $status$ indicates the
          goal status of propositions: $0$ for non-goal propositions in ${\cal I}\setminus {\cal G}$, $1$ for unachieved goal propositions in ${\cal G}\setminus {\cal I}$, and $2$ for achieved goal propositions in ${\cal I} \cap {\cal G}$. $status=3$ for object vertices.
        \item $class$ refers to the object type for object vertices, and for proposition vertices, it denotes the predicate of which the proposition is an instance.
        \end{itemize}
    \item $\ecol: E \rightarrow \mathbb{N}$, where for each edge $e \in E$, $\ecol(e)$ indicates the index of the object in the proposition parameters.
\end{itemize}
\end{definition}

In ILG, the object type information is absent, whereas \pog\ embeds it within each object
vertex. This may seem minor, but it adds valuable information to each object
vertex significantly enriching the information available. Moreover, Fast Downward omits static propositions during search, which causes them to be missing in existing ILG implementations as well. While this omission does not affect traditional heuristic methods, it significantly impacts learning methods, which estimate heuristics based on the graph. Without static propositions, crucial information is lost, leading to blind guesses for some actions. For instance, ``waiting'' propositions in the childsnack domain are static, and without this information, planners are unable to determine which table the tray should be moved to. Therefore, \pog\ includes static propositions.

In the following, all elements of the problem $\Pi$ are fixed, except for the current state $s$. \textbf{We shall therefore identify $\Pi$ with $s$ and will refer to the TILG $G_\pi$ as $G_s$.}

\subsection{Action Pruning}
With Greedy Best-First Search (GBFS), planners select the state with the smallest heuristic value to expand, which requires computing the heuristic value of all child states. When child states contain a large number of symmetries, this can result in significant time wasted on redundant calculations.

Shallow pruning was designed to address this challenge \cite{pochter:etal:11}. However, the problem description graph (PDG) used in shallow pruning requires instantiating all predicates and action schemas within the graph, resulting in significant computational overhead for each state. To improve efficiency, we introduce Action Pruning, which replaces the PDG with \pog. A key innovation of action pruning is its ability to \textbf{infer symmetrical child states from the parent state}, eliminating the need for action preconditions and effects in the graph. By leveraging the much more compact \pog\ representation and its inference capability, action pruning enables faster automorphism calculations.

\begin{definition}(Object Tuples Equivalence)
  Let $\langle A_1,\dots,A_n \rangle$ and $\langle B_1,\dots,B_n \rangle$ be two tuples of objects s.t. $A_i$ and $B_i$ in $\mathcal{O}$ with corresponding vertices $u_i$ and $v_i$ in the \pog\ $G_s$. We say that $\langle A_1,\dots,A_n \rangle$ is equivalent to $\langle B_1,\dots,B_n \rangle$ in $s$, denoted as $\langle A_1,\dots,A_n \rangle \simeq \langle B_1,\dots,B_n \rangle$, iff there exist an automorphism of $G_s$ represented by the bijective function $\sigma$ s.t. $\sigma(u_i) = v_i$ for all $i \in \{1,\dots,n\}$.
\end{definition}

\begin{theorem}\label{main}
Let $A_i$ and $B_i$ be objects in ${\cal O}$ for $i \in \{1, \dots, n\}$ with $A_i \neq A_j$ and $B_i \neq B_j$ for all $i\neq j$. Let an action schema $\alpha \in \mathcal{A}$, and consider two ground actions, $a = \alpha(A_1, A_2, \dots, A_n)$ and $b = \alpha(B_1, B_2, \dots, B_n)$, applicable in a state $s$, resulting in successor states $s_a$ and $s_b$ respectively. If $\langle A_1,\dots,A_n \rangle \simeq \langle B_1,\dots,B_n \rangle$ in $s$ then the TILGs $G_{s_a}$ and $G_{s_b}$ are isomorphic, i.e., $G_{s_a} \cong G_{s_b}$.

\end{theorem}
The proof of Theorem~\ref{main} is in the supplementary material.
\stcomment{Since we have space, I wonder whether it would be possible to put a little proof sketch? That would be a better use of the 300 USD!}

\begin{algorithm}[!tp]
\caption{Action pruning algorithm}
\label{alg:algorithm}
\begin{algorithmic}[1] 
\renewcommand{\algorithmicrequire}{\textbf{Input:}}
\renewcommand{\algorithmicensure}{\textbf{Output:}}
\REQUIRE Planning problem with current state $s$
\REQUIRE Set $A_s$ of actions applicable in $s$
\ENSURE Pruned action set $A_p \subseteq A_s$

\STATE $K \gets \emptyset$, $A_p \gets \emptyset$
\STATE Graph $G_s \gets \text{\pog}(s)\ with\ encoding\ in\ Eq.\ 1$
\STATE Orbits $O_{s} \gets \text{Nauty}(G_s)$

\FOR{$a$ \textbf{in} $A_s$}
    \STATE $K_{a} \gets \text{Replace\_params\_with\_orbits}(a, O_s)$
    \IF{$K_{a}$ \textbf{not in} $K$}
        \STATE $K \gets K \cup \{K_{a}\}$
        \STATE $A_p \gets A_p \cup \{a\}$
    \ENDIF
\ENDFOR
\RETURN $A_p$
\end{algorithmic}
\end{algorithm}
Unfortunately, identifying all isomorphic successor states in order to prune actions requires testing an exponential number of tuples against the automorphisms of $G_s$.
Therefore, we resort to a simpler
method that over-approximates the set of equivalent tuples and isomorphic successor states, and consequently does not preserve the completeness of the search process. 
It relaxes the conditions of the theorem by checking the equivalence of all individual pairs $A_i$ and $B_i$ in $s$, i.e., the condition that $\langle A_1,\dots,A_n\rangle \simeq \langle B_1,\dots,B_n \rangle$ is replaced with $\langle A_i \rangle \simeq \langle B_i \rangle ~ \forall i$, or in other terms that $A_i$ and $B_i$ are in the same orbit of $G_s$.
In our experiments, we find that this over-approximation yields good results in practice. Moreover, we did not observe any failure due to incompleteness, and therefore do not currently employ any fallback mechanism.

\stcomment{Yingbin, note that after discussion with Felipe, I removed ``colored'' to avoid confusion with the colored graph definitions in the background. I also amended the corresponding line in Alg. 1}

This process of action pruning is outlined in Algorithm \ref{alg:algorithm}.
First, the planning problem with current state $s$ is converted into a
\pog\ $G_s$. The Nauty library \cite{mckay:piperno:14} is then utilized to compute the orbits $O_s$ of $G_s$.
Since Nauty lacks support for feature vectors, we aggregate
vertex features into an unique
vertex color to detect automorphisms. This color-coding strategy is detailed in Equation \ref{eq:color}:
\vspace*{-2ex}
\begin{equation}
\label{eq:color}
\begin{split}
\text{color} & = \sum_{i=1}^N 10^{\beta_i}  \times F_i \quad \mbox{with}\\
    \beta_i & = \left\{ \begin{array}{cl}
                    \sum_{n=1}^{i - 1} \lceil \log_{10} M_n \rceil &  i \geq 2 \\
                    0 &  i = 1
                   \end{array} \right. ,
\end{split}
\end{equation}
where $N$ denotes the number of features, $F_i$ represents the value of feature $i$, and $M_n$ is the maximum possible value of feature $i$.

After obtaining the orbits $O_s$ of $G_s$, the parameters of each applicable action $a$ are substituted with their respective orbit IDs, generating a unique hash key $K_{a}$. This hash key is subsequently used to identify and eliminate symmetric actions, ensuring that only distinct actions are retained in $A_p$ for further processing.

\subsection{State Pruning}

Symmetries arise not only between child states but also across states from different parents. Many state pruning approaches have been proposed and proven useful in classical planning \cite{pochter:etal:11, domshlak:etal:12}. However, the main issue limiting their widespread use in planning problems is their high computational cost. To address this issue, we propose a novel method that performs state pruning with negligible additional overhead. Specifically, building on the permutation invariance property of GNNs, we use the embeddings from the second-to-last layer of the network as hash keys to efficiently detect and eliminate symmetries across states.

The idea of using neural network outputs to check similarity is not new, having been employed in Siamese networks since early work in deep learning \cite{bromley:etal:93}. These identical architecture, weight-shared networks are specifically designed to assess and compare the similarity between two inputs. This approach has proven effective across various fields, including fingerprint identification \cite{li:etal:21} and anomaly detection \cite{zhou:etal:21}. For GNNs, \citet{chen:etal:19} highlight the equivalence between graph isomorphism testing and approximating permutation-invariant functions. Moreover, standard GNNs have been shown to possess an expressive power, comparable to that of the 1-LWL test \cite{xu:etal:19}. While this implies GNNs may be unable to distinguish some non-isomorphic graphs, compromising the completeness of the search when state pruning is used, our results demonstrate that GNNs based on \pog\ can be highly effective in both heuristic prediction and state pruning.

The \pog\ $G_{s_i}$ for the current state $s_i$ is fed through a graph network $\phi_{\theta}$ to encode an embedding $z_i$.
Subsequently, $z_i$ is processed by a fully connected linear layer $\varphi_{\theta}$ to generate a heuristic value $\hat{h}_i \in \mathbb{R}$. This process is represented by
$z_{i} = \phi_{\theta}(G_{s_i})$ and
$\hat{h}_i = \varphi_{\theta}(z_{i})$.

Next, $z_i$ is rounded up and encoded using MD5 to shorten its length, serving as a key in a hash map for state matching. Since $z_i$ is efficiently captured during the network's forward pass, there is no need to generate keys through computationally expensive methods like calculating isomorphisms in PDG \cite{pochter:etal:11}, resulting in minimal additional cost for state pruning.

\section{Experiments}

\begin{table}[!tp]
\centering
\resizebox{\linewidth}{!}{%
\setlength{\tabcolsep}{2pt}
\begin{tabular}{lccccccccc}
\toprule
Domain       & $h^{FF}$  & LAMA    & GOOSE & OptRank & GPR & Distincter\\
\midrule
blocksworld  & 28        & 61        & 61$\pm$10   & 44$\pm$11& 69        & \textbf{88$\pm$4}\\
childsnack   & 26        & 34        & 16$\pm$4    & 32$\pm$1 & 20        & \textbf{64$\pm$5}\\
ferry        & 71        & 70        & 70$\pm$0    & 64$\pm$4 & 82        & \textbf{83$\pm$1}\\
floortile    &\textbf{10}&\textbf{10}& 1$\pm$0     & 1$\pm$0  & 2         & 2$\pm$0\\
miconic      &\textbf{90}&\textbf{90}& 89$\pm$1    & 88$\pm$4 &\textbf{90}& \textbf{90$\pm$0}\\
rovers       & 29        &\textbf{70}& 28$\pm$1    & 31$\pm$2 & 36        & 42$\pm$2\\
satellite    & 64        &\textbf{90}& 29$\pm$2    & 29$\pm$3 & 39        & 48$\pm$17 \\
sokoban      & 36        &\textbf{40}& 34$\pm$0    & 32$\pm$1 & 38        & 32$\pm$2 \\
spanner      & 30        & 30        & 39$\pm$16   & 65$\pm$0 & 74        & \textbf{90$\pm$0}\\
transport    & 41        &\textbf{68}& 37$\pm$4    & 42$\pm$5 & 28        & 50$\pm$3 \\
\midrule
Sum         & 425        & 563       & 405         & 429      & 478       & \textbf{589}\\
\bottomrule
\end{tabular}
}
\caption{Coverage comparison with SOTA methods on the 2023 International Planning Competition Learning Track.}
\label{tab:total_coverage}
\end{table}

\begin{table}[!t]
    \centering
    \begin{tabular}{l c c}
        \toprule
         Domain     & LAMA        & Distincter          \\ \midrule
        blocksworld & 390         & \textbf{198$\pm$10} \\
        childsnack  & 45          & \textbf{34$\pm$3}   \\
        ferry       & 257         & \textbf{206$\pm$2}  \\
        floortile   & 34          & \textbf{32$\pm$0}   \\
        miconic     & 324         & \textbf{273$\pm$8}  \\
        rovers      & \textbf{72} & 106$\pm$11          \\
        satellite   & \textbf{18} & 27$\pm$8           \\ 
        sokoban     & \textbf{46} & 49$\pm$14          \\ 
        spanner     & \textbf{14} & 16$\pm$0           \\ 
        transport   & 49          & \textbf{45$\pm$3}  \\ \midrule
        Sum         & 1249        & \textbf{987}        \\ \bottomrule
    \end{tabular}
    \caption{Average plan lengths over problems solved by both LAMA and Distincter.}
    \label{tab:lengths_table}
\end{table}

\begin{table}[!t]
\centering
\begin{tabular}{lcccc}
\toprule
Domain          & None      & Action   & State    & Distincter \\
\midrule
blocksworld     & 79$\pm$11 & 79$\pm$7 & 88$\pm$3 & 88$\pm$4 \\
childsnack      & 34$\pm$4  & 63$\pm$3 & 61$\pm$1 & 64$\pm$5 \\
ferry           & 82$\pm$0  & 82$\pm$0 & 83$\pm$1 & 83$\pm$1 \\
floortile       &  2$\pm$1  &  2$\pm$0 &  2$\pm$0 &  2$\pm$0 \\
miconic         & 90$\pm$0  & 90$\pm$0 & 90$\pm$0 & 90$\pm$0 \\
rovers          & 41$\pm$2  & 41$\pm$3 & 41$\pm$2 & 42$\pm$2 \\
satellite       & 45$\pm$13 & 46$\pm$13& 47$\pm$17& 48$\pm$17\\
sokoban         & 32$\pm$2  & 32$\pm$2 & 32$\pm$2 & 32$\pm$2 \\
spanner         & 83$\pm$0  & 90$\pm$0 & 83$\pm$0 & 90$\pm$0 \\
transport       & 42$\pm$1  & 42$\pm$1 & 49$\pm$2 & 50$\pm$3 \\
gripper         & 24$\pm$9  & 75$\pm$4 & 90$\pm$0 & 90$\pm$0 \\
grippers        & 62$\pm$5  & 89$\pm$1 & 85$\pm$1 & 90$\pm$1 \\
logistics       & 19$\pm$9  & 36$\pm$4 & 53$\pm$4 & 52$\pm$3 \\
movie           & 90$\pm$0  & 74$\pm$17& 90$\pm$0 & 75$\pm$17\\
tsp             & 76$\pm$24 & 78$\pm$0 & 90$\pm$0 & 90$\pm$0 \\
tyreworld       &  0$\pm$0  &  1$\pm$1 & 64$\pm$20& 65$\pm$21\\
\midrule
Sum             & 803       & 920      & 1048     & 1051     \\
\bottomrule
\end{tabular}
\caption{Ablation study. ``None'' refers to GBFS + GNN heuristic without pruning, ``Action'' denotes the use of action pruning, ``State'' represents the use of state pruning.
}
\label{tab:ablation}
\end{table}

\subsubsection{Datasets.} We evaluate our framework, Distincter, on the 2023 International Planning
Competition Learning Track \cite{seipp:segovia:23}, which includes ten domains. In ablation experiments, to assess the effectiveness of two proposed pruning methods, we further employ six domains that contain a lot of symmetrical states.

\subsubsection{Network structure.} Our graph network consists of RGCN layers with a hidden dimension of 64 \cite{schlichtkrull:etal:18}, followed by global add pooling, and a linear layer producing a one-dimensional output. The network is implemented using the standard PyTorch Geometric package \cite{fey:lenssen:19}. For further setting information, please see the supplementary material.

\subsubsection{Training and evaluation.} For each domain, we train a GNN using the RMSE loss function for 30 epochs, including 10 warm-up epochs. We use an initial learning rate of $10^{-3}$ and apply cosine annealing\cite{loshchilov:hutter:17} over a single cycle, with a momentum value of $0.9$. To adapt to varying numbers of examples ($N$) across different domains, we set the number of iterations to 100 per epoch and adjust the batch size accordingly, using $\frac{N}{100}$.

For evaluation, our planner is based on Fast Downward using eager-GBFS \cite{helmert:06} guided by GNN heuristic values. Upon completing the training, the model is saved in JIT format and executed using C++. To ensure result stability and mitigate dataset bias, we employ early stopping on a validation set to select optimal models \cite{bai:etal:23}. All experiments are conducted on a single CPU core with an NVIDIA A6000 GPU and 8GB of memory, with a 30 minute timeout per problem. The mean and standard deviation are computed from three trials.

\subsection{Results}

We compare Distincter with SOTA baselines, including both traditional heuristic search methods, namely LAMA \cite{richter:westphal:10} and GBFS with $h^{FF}$~\cite{hoffmann:nebel:01}, and GBFS with learnt heuristics using
GOOSE \cite{chen:etal:24:aaai}, its
optimal ranking counterpart
\cite{hao:etal:24}
and Gaussian Process Regression (GPR)
from WL features \cite{chen:etal:24:icaps}.
Note that, as shown in \cite{hao:etal:24}, other methods such
as STRIPS-HGN \cite{shen:etal:20} and Perfrank \cite{chrestien:etal:23}, are dominated by our baselines. All baselines are run on the same hardware and with the same computational requirements as Distincter.
In terms of coverage, Distincter matches or surpasses all baselines across five domains. Notably, when compared to the strongest learning baseline, GPR \cite{chen:etal:24:icaps}, Distincter achieves parity or superiority in eight of the ten domains. Additionally, the total coverage of Distincter surpasses that of model-based methods: it exceeds that of $h^{FF}$ by a substantial margin of 164 and 
that of LAMA by 26.
In Table~\ref{tab:lengths_table}, we report the average plan length over the problems successfully solved by both approaches. The results suggest a correlation between plan length and coverage.

In keeping with the 8GB memory requirement of the IPC learning track, we found it is insufficient for the Fast Downward translator to ground some large problems, resulting in performance loss as shown in Table \ref{tab:total_coverage}. For instance, in the ``ferry'' domain, when sufficient memory is available, Distincter can solve all 90 test problems.

Although Distincter exhibits 
very good performance across many domains, it struggles in others -- see e.g. its low performance on ``Floortile'' and ``Sokoban'', and the large deviations observed in the ``Satellite'' domain. The key issue with Floortile and Sokoban is that they require path-finding and geometrical reasoning which cannot be achieved with the limited receptive field of ordinary GNNs. Dead ends in these domains are another issue, as they cannot always be captured when training with optimal plans only. In Satellite, due to the the lack of static propositions in its graph, GOOSE learns a simple strategy that works in simple problems only. On the other hand, thanks to statics being included, Distincter is able to learn to solve much larger problems. However, standard GNNs are not expressive enough to learn to distinguish all non-isomorphic Satellite states \cite{drexler:etal:24:kr},
and therefore the learning procedure fails every once in a while, leading to a large variance.

\subsection{Ablation Study}

To assess the effectiveness of our proposed pruning techniques, we performed ablation experiments across four configurations. From Table \ref{tab:ablation}, we observe that in domains with a high degree of symmetries, such as ``spanner'' and ``childsnack'', action pruning offers substantial benefits. However, in the ``movie'' domain, which contains a large number of redundant objects, action pruning consumes excessive time computing symmetries. In contrast, state pruning improves performance in seven domains, demonstrating its broader utility. By combining both techniques, we can leverage their strengths to achieve better overall outcomes.

\section{Conclusion}

We introduced \pog, a novel graphical representation that captures key problem-solving features and is designed for combining efficiency with symmetry detection. Leveraging the properties of \pog, we proposed two efficient pruning techniques that are suitable for large-scale planning problems. Our framework, Distincter, achieved a historic milestone by outperforming the LAMA framework on the learning track of the 2023 International Planning Competition.

In addition, both pruning methods are applicable to traditional model-based approaches. Although state pruning with GNNs can be computationally expensive, small dedicated GNNs can mitigate this issue.

\section{Acknowledgments}
The authors thank Sandra Kiefer and Brendan McKay for useful discussions. This work was supported by the Australian Research Council grant DP220103815, by the Artificial and Natural Intelligence Toulouse Institute (ANITI) under the grant agreement ANR-23-IACL-0002, and by the European Union’s Horizon Europe Research and Innovation program under the grant agreement TUPLES No. 101070149.

\bibliography{aaai25}

\newpage
\appendix

\section{Proof of Theorem 1}

\newcommand{\predargk}[1]{\ensuremath{p(-,#1,-)}}
\newcommand{\tilg}{\ensuremath{g}}
\newcommand{\edge}[1]{\langle #1 \rangle}

\begin{lemma}
Let $\sigma$ be any automorphism function of $G_s$ (the TILG of state $s$) satisfying $\langle A_1,\dots,A_n \rangle \simeq \langle B_1,\dots,B_n \rangle$ where $A_i$ and $B_i$ are objects in $\mathcal{O}$ and, $A_i \neq A_j$ and $B_i \neq B_j$ for all $i\neq j$.
Then $\sigma(p(A_{i_1}, \dots, A_{i_k})) = p(B_{i_1}, \dots, B_{i_k})$ for all propositions $p(A_{i_1}, \dots, A_{i_k}) \in s$ where $i_1, \dots, i_k$ is a permutation of a subset of size $k$ \mbox{of $\{1,\dots, n\}$.}
\end{lemma}

\begin{proof}
By the definition of $\langle A_1,\dots,A_n \rangle \simeq \langle B_1,\dots,B_n \rangle$, we have that $\sigma(A_{i_j}) = B_{i_j}$.
Moreover, $\sigma$ preserves vertex and edge colors, therefore $\sigma(p(A_{i_1}, \dots, A_{i_k}))$ must be a proposition from predicate $p$ (vertex class color) and, for all $i_j \in \{1,\cdots,n\}$, its $i_j$-th argument must be $\sigma(A_{i_j}) = B_{i_j}$ because the edges $(A_{i_j}, p(A_{i_1}, \dots, A_{i_k}))$ and $(\sigma(A_{i_j})= B_{i_j}, \sigma(p(A_{i_1}, \dots, A_{i_k})))$ have the same color $i_j$.
\end{proof}

\begin{theorem}
Let $A_i$ and $B_i$ be objects in ${\cal O}$ for $i \in \{1, \dots, n\}$ with $A_i \neq A_j$ and $B_i \neq B_j$ for all $i\neq j$. Let an action schema $\alpha \in \mathcal{A}$, and consider two ground actions, $a = \alpha(A_1, A_2, \dots, A_n)$ and $b = \alpha(B_1, B_2, \dots, B_n)$, applicable in a state $s$, resulting in successor states $s_a$ and $s_b$ respectively. If $\langle A_1,\dots,A_n \rangle \simeq \langle B_1,\dots,B_n \rangle$ in $s$ then the TILGs $G_{s_a}$ and $G_{s_b}$ are isomorphic, i.e., $G_{s_a} \cong G_{s_b}$.

\end{theorem}

\begin{proof}
Let $G_{s} = \langle V_s, E_s, f_s, l_s \rangle$ be the TILG for state $s$ and $V_s = \{o_1,\dots,o_n\} \cup \{p^s_1,\dots,p^s_m\} \cup \{g_1,\dots,g_k\}$ where $o_i$ are objects, $p^s_i$ are propositions true in $s$ and $g_i$ are goal propositions.
Also, let $a^+$ and $a^-$ ($b^+$ and $b^-$) represent the ordered list of add and delete effects of action $a$ ($b$).
Since $a$ and $b$ are groundings of $\alpha$, we have that $|a^+| = |b^+|$ and $|a^-| = |b^-|$ and that the $i$-th item in $a^+$ and $b^+$ ($a^-$ and $b^-$) correspond to the same predicates $p_i$ with different object instantiations.

Let $G_{s_a} = \langle V_a, E_a, \ncol_a, \ecol_a \rangle$ and $G_{s_b} = \langle V_b, E_b, \ncol_b, \ecol_b \rangle$ be the TILG associated with $s_a$ and $s_b$, respectively.
First, let's partition $V_a = V^\text{old}_a \cup V^\text{new}_a$ where $V^\text{old}_a$ are the vertices also in $V_s$ and $V^\text{new}_a$ are the new vertices in $V_a$ wrt $V_s$.
Clearly, $V^\text{old}_a = V_s \cap V_a$ and, applying the TILG definition and the STRIPS rule, we have $V^\text{old}_a = \mathcal{O} \cup \mathcal{G} \cup (s \cap ((s \setminus a^-) \cup a^+))$.
With more algebra manipulation, we arrive at $V^\text{old}_a = \mathcal{O} \cup \mathcal{G} \cup (s \setminus a^-) \cup (s \cap a^+)$, i.e., all objects, all goal propositions, all propositions not deleted and all added propositions that were already in $s$.
Doing the same for $V^\text{new}_a = V_a \setminus V_s$, we have $V^\text{new}_a = (\mathcal{O} \cup \mathcal{G} \cup s_a) \setminus (\mathcal{O} \cup \mathcal{G} \cup s) = s_a \setminus (\mathcal{O} \cup \mathcal{G} \cup s)$.
Since $s_a \cap \mathcal{O} = \emptyset$, after applying the STRIPS rule, we have $V^\text{new}_a = ((s \setminus a^-) \cup a^+) \setminus (\mathcal{G} \cup s) = a^+ \setminus (\mathcal{G} \cup s)$, i.e., all add effects that are not a goal proposition nor already in $s$.
Notice that all elements in $V^\text{new}_a$ are propositions of the form $p^a_i(A_{i1}, A_{i2}, \dots, A_{ik_i})$ and that the same can be done to create a partition of $V_b$.

In order to show that $G_{s_a}$ and $G_{s_b}$ are isomorphic, we need to show that there exists a function $\iso: V_a \to V_b$ that is a bijection and also preserves the vertex and edge colors.
We show that the following function $\iso$ satisfies all these requisites:
\begin{itemize}
\item $\iso(v) = \sigma(v)$ for $v \in V^\text{old}_a$
\item $\iso(p^a_i(A_{i1}, A_{i2}, \dots, A_{ik_i})) = p^b_i(B_{i1}, B_{i2}, \dots, B_{ik_i})$ for all $p^a_i(A_{i1}, A_{i2}, \dots, A_{ik_i}) \in V^\text{new}_a$ where $p^b_i(B_{i1}, B_{i2}, \dots, B_{ik_i})$ is the respective added proposition in $b^+$ based on the action schema grounding
\end{itemize}
where $\sigma$ is any automorphism function of $G_s$ satisfying $\langle A_1,\dots,A_n \rangle \simeq \langle B_1,\dots,B_n \rangle$.

\paragraph{$\iso$ is a bijection.}
For all $p^a_i(A_{i1}, A_{i2}, \dots, A_{ik_i}) \in V^\text{new}_a$, $\iso(p^a_i(A_{i1}, A_{i2}, \dots, A_{ik_i})) = p^b_i(B_{i1}, B_{i2}, \dots, B_{ik_i})$ is a bijection from $V^\text{new}_a$ to $V^\text{new}_b$ since it is a 1-to-1 mapping by construction and $p^b_i(B_{i1}, B_{i2}, \dots, B_{ik_i}) \not\in V^\text{old}_b$ otherwise Lemma 1 would imply that $p^a_i(A_{i1}, A_{i2}, \dots, A_{ik_i}) \in V^\text{old}_a$ which is impossible since it is in $V^\text{new}_a$ by assumption.

Since $\sigma$ is an automorphism, it is a bijection of $V_s$ to $V_s$ therefore, for all $v \in V^\text{old}_a$, $\iso(v) \in V_s$ and the same applies for $v \in V^\text{old}_b$.
Thus we need to show that $\sigma$ is also a bijection between $V^\text{old}_a$ and $V^\text{old}_b$.
Since $\sigma$ represents an automorphism for a TILG, it preserves the vertex and edge labels, therefore, $\sigma$ is a bijection between $X$ to $X$ where $X \subseteq V$ is the set of vertices with same colors in $G_s$.
Thus $\sigma$ is a bijection between all $v \in \mathcal{O}$ and between all $g \in \mathcal{G}$ since they are in all TILGs.

We still need to show the bijection between all the vertices representing propositions already in $s$ that remained true in $s_a$ and in $s_b$.
Namely, we need to show the bijection between $s \setminus (a^- \setminus a^+ )$ and $s \setminus (b^- \setminus b^+)$.
This follows from Lemma~1, therefore $\iso$ is a bijection.

\paragraph{Vertex color compatibility.}
We need to show that, for all $v \in V_a$, the color of $v$ in $G_{s_a}$ equals the color of $\iso(v)$ in $G_{s_b}$, i.e., $\ncol_a(v) = \ncol_b(\iso(v))$.
For all $v \in V^\text{new}_a$, $v$ represents a proposition by definition and let $p$ be the predicate that generate it.
The only possible color for $v$ is ``($p$, non-goal)'' since $G_{s_a}$ is a valid TILG, i.e., $v$ cannot be mentioned in the goal otherwise it would be in $V^\text{old}_a$.
The same argument applies to $V^\text{new}_b$.
As shown in the bijection part of the proof, $\iso$ is a bijection between $V^\text{new}_a$ and $V^\text{new}_b$, therefore we have $\ncol_a(v) = \ncol_b(\iso(v))$ for all $v \in V^\text{new}_a$.

We still need to show color compatibility for all $v \in V^\text{old}_a = \mathcal{O} \cup \mathcal{G} \cup (s \setminus a^-) \cup (s \cap a^+)$.
This is trivially true for $v \in \mathcal{O}$ since objects can only be labelled ``(obj, obj-type)'' and their object type is immutable.
For all $v \in (s \setminus a^-) \cup (s \cap a^+) \setminus \mathcal{G}$, i.e., non-goal propositions that are true in both $s$ and $s_a$, we have that $\ncol_a(v) = \ncol_b(\iso(v))$ because $\iso(v) = \sigma(v)$ and $\sigma$ is color compatible.

The remaining case is $v \in \mathcal{G}$.
Let $p$ be the predicate that generate the proposition $v$.
Since $G_{s_a}$ is a valid TILG, $\ncol_a(v)$ is either ``($p$, unachieved-goal)'' or  ``($p$, achieved-goal)''.
If $\ncol(v) = \ncol_a(v)$, i.e., the color of $v$ was not affect by $a$, then we have $\ncol(v) = \ncol(\sigma(v)) = \ncol_a(v) = \ncol_b(\tau(v))$ because $\sigma(v) = \tau(v)$ and $\sigma$ is color compatible.

Otherwise we have that $\ncol(v) \neq \ncol_a(v)$ and the color of $v$ was changed by $a$.
Since $a$ and $b$ are actions from the same action schema, i.e., $a = \alpha(A_1, A_2, \dots, A_n)$ and $b = \alpha(B_1, B_2, \dots, B_n)$ and $\langle A_1,\dots,A_n \rangle \simeq \langle B_1,\dots,B_n \rangle$, then $v \in a^+$ iff $\sigma(v) \in b^+$ and $v \in a^-$ iff $\sigma(v) \in b^-$ due to the action schema grounding.
Moreover, $\iso(v) = \sigma(v) \in \mathcal{G}$ and $\ncol(v) = \ncol(\sigma(v))$, i.e., both $v$ and $\iso(v)$ are in goal propositions with the same color in $G_s$ denoting if they are both achieved or both non-achieved.
Therefore, the change in the label of $v$ will be the same after applying $a$ in $s$ and applying $b$ in $s$ and $\ncol_a(v) = \ncol_b(\tau(v))$.

\paragraph{Edge color compatibility.}
TILGs are bipartite undirected graphs, thus all edges can be represented as an object-proposition pair with the color representing the position of the object in the instantiation of the predicate resulting in the given proposition.
Thus we need to show that, for all $(o, p) \in E_a$, $\ecol_a((o,p))$ and $\ecol_b((\iso(o),\iso(p)))$ are the same.
Notice that, $o \in \mathcal{O}$ is always in $V^\text{old}_a$ and in $V^\text{old}_b$.
If $p \in V^\text{old}_a$, then $\ecol_b((\iso(o),\iso(p))) = \ecol_b((\sigma(o),\sigma(p)))$ by the definition of $\iso$ and it equals $\ecol_a((o,p))$ since $\sigma$ is an automorphism.
Otherwise, $p \in V^\text{new}_a$ and, without loss of generality, let $p$ be $p^a_i(A_{i1}, A_{i2}, \dots, A_{ik_i})$ and $o$ be $A_{ij}$ for $j \in \{1,\dots,k_i\}$.
Since both $G_a$ and $G_b$ are valid TILGs, we have $\ecol_b((\iso(A_{ij}),\iso(p^a_i(A_{i1}, A_{i2}, \dots, A_{ik_i})))) = \ecol_b((B_{ij},p^b_i(B_{i1}, B_{i2}, \dots, B_{ik_i}))) = j$ which is equals $\ecol_a((A_{ij},p^a_i(A_{i1}, A_{i2}, \dots, A_{ik_i}))) = j$, i.e., $A_{ij}$ and $B_{ij}$ are the $j$-th argument of $p^a_i(A_{i1}, A_{i2}, \dots, A_{ik_i})$ and $p^b_i(B_{i1}, B_{i2}, \dots, B_{ik_i})$, respectively.
Therefore, $\iso$ is edge color compatible.
\end{proof}

\begin{figure*}[!tp]
    \centering
    \begin{subfigure}[b]{0.49\textwidth}
        \centering
        \includegraphics*[clip=true, width=\linewidth]{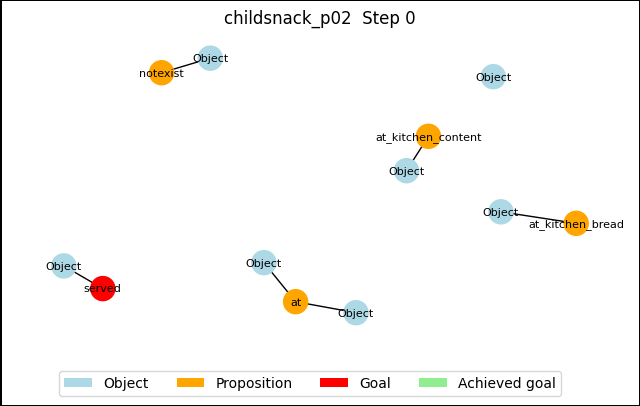}
        \caption{ILG}
    \end{subfigure}\hfill
    \begin{subfigure}[b]{0.49\textwidth}
        \centering
        \includegraphics*[clip=true, width=\linewidth]{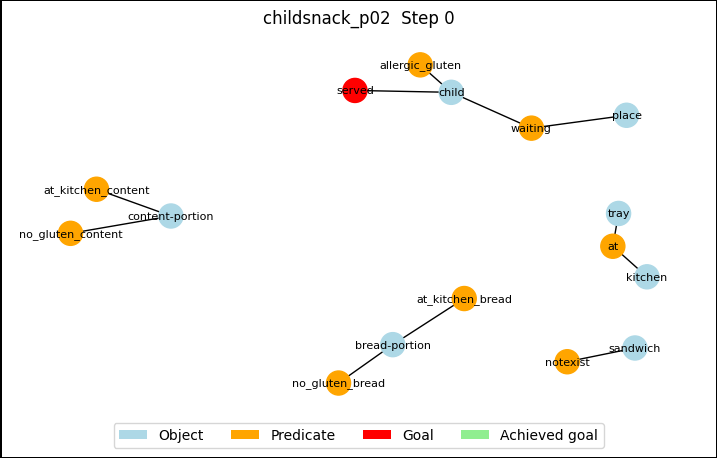}
        \caption{\pog}
    \end{subfigure}\hfill
    \caption{A planning problem in the childsnack domain is illustrated using both ILG and \pog. Note that the colors used are for visualization purposes only, differing from the color encoding employed for computing automorphisms. Observe ILG lacks types (besides ``Object'') and static propositions.}
    \label{fig:tilg}
\end{figure*}

\begin{table*}[!th]
\centering
\fontsize{9}{11}\selectfont
\begin{tabular}{lccccc}
\toprule
\multirow{2}{*}{Domain}  & \multicolumn{3}{c}{Training} & \multicolumn{2}{c}{Testing}  \\
                & Parameters      & Num. & Solved     & Parameters         & Num.  \\
\midrule
gripper         & $n \in [1, 30]$ & 30 & 16         & $n \in [50, 800]$  & 90     \\
grippers        & $n \in [1, 2],\ r \in [1, 3],\ o \in [4, 16]$ & 30 & 30     & $n \in [5, 50],\ r \in [3, 5],\ o \in [20, 400]$  & 90   \\
logistics       & $a \in [1, 3],\ c \in [2, 3],\ s \in [2, 3],\ p \in [2, 32]$   & 50 & 40     & $a \in [1, 5],\ c \in [2, 5],\ s \in [2, 3],\ p \in [5, 100]$  & 90   \\
movie           & $n \in [1, 30]$ & 30 & 30         & $n \in [50, 800]$  & 90     \\
tsp             & $n \in [1, 30]$ & 30 & 30         & $n \in [50, 1600]$ & 90     \\
tyreworld       & $n \in [1, 30]$ & 30 & 17         & $n \in [11, 100]$  & 90     \\
\bottomrule
\end{tabular}
\caption{The parameters are used to generate training and testing problems. The 'Solved' column indicates the number of problems solved by the Scorpion planner \cite{Seipp17Scorpion}.}
\label{tab:additional-domains}
\end{table*}

\section{Related Work}
\textbf{Symmetries in planning} Identifying symmetries during search is crucial for pruning search spaces, and numerous studies have investigated this area \cite{domshlak:etal:12, Shleyfman15, sievers:etal:19:icaps}. \citet{pochter:etal:11} introduce the Problem Description Graph (PDG), a structure that integrates all variables, their possible values, and action preconditions and effects
to compute automorphism generators at the initial state. These generators are key to creating canonical forms that aid in identifying symmetries during the search process. The paper also introduces a shallow pruning method, which eliminates branches without producing child states. While both methods share similarities with ours, our methods are notably more efficient. For state pruning, our method eliminates the need for automorphism calculations and canonical transformations, substantially reducing computational overhead. Action pruning focuses exclusively on the current state and
omits extraneous values and actions, significantly reducing graph size and consequently, the automorphism computation time. Furthermore, PDG requires fixing the number of objects \cite{Correa24ObjectCreation} to find all generators, while \pog \ is based on each state, so it can adjust this change.

\citet{sievers:etal:19:aaai} adopts graphs normally used for symmetry detection, such as the PDG or the Abstract Structure Graph (ASG) \cite{sievers:etal:19:icaps}, to learn planning portfolios.
However, these graphs are transformed into grayscale images for processing by a convolutional network. Moreover, the ASG, which encodes a lifted planning problem, creates computational issues for learning with GNNs, as its encoding of predicate and action schema arguments via a directed path of length linear in the number of arguments requires a large receptive field, and its directed edges restricts the information flow. Recent independent work by \cite{drexler:etal:24:icaps} proposes the Object Graph to represent states and detect all state symmetries by using graph isomorphism algorithms. While this approach effectively reduces symmetries in the training data, the time required to compute isomorphisms makes it impractical during the search process. In contrast, \pog \ provides a more compact structure than the Object Graph. By using edge types to represent parameter indices and consolidating each proposition into a single vertex, \pog \ is both clearer and contains fewer vertices.

\noindent \textbf{Graph neural networks in planning} The widespread success of graph neural networks in various fields has motivated an increasing number of researchers to explore neural networks for planning problems \cite{stahlberg:etal:22, horcik:sir:24, hao:etal:24, drexler:etal:24:kr}.

Chen et al. \cite{chen:etal:24:aaai, chen:etal:24:icaps} introduce two graph structures: the Lifted Learning Graph (LLG) and the Instance Learning Graph (ILG). LLG includes objects, propositions, goals, and action schemas to compute domain-independent heuristic values, while ILG simplifies LLG by excluding action schemas, which reduces graph size and increases processing speed in the domain-dependent setting. Our work builds upon these foundations by introducing \capsfullpog\ (\pog), which significantly enhances the graph-level representation with object types and static propositions. As illustrated in Figure \ref{fig:tilg}, a direct comparison between two graphs reveals that ILG has significantly fewer vertices and connections compared to \pog, resulting in a much smaller graph. The absence of object types and static propositions and their connections with objects in ILG may restrict the ability of neural networks to effectively learn from data.

\section{Additional Experimental Setting}
\subsection{Data Generation}
The 2023 International Planning Competition Learning Track \cite{seipp:segovia:23} is used in our experiments.
It consists of 10 domains, each of which contains 100 training problems and 90 testing problems. The testing problems are categorized into three levels of difficulty: easy, medium, and hard, with 30 problems in each category. To further evaluate the proposed pruning methods, we employ six additional domains that contain a high number of symmetrical states. The problems for these domains are generated using PDDL generators \cite{seipp-et-al-zenodo2022}. To
replicate the setting of the Learning Track, we generate training problems and testing problems for each additional domain (see Table~\ref{tab:additional-domains}).

For training data, we follow the setting of GOOSE \cite{chen:etal:24:aaai}, and solve the training problems using the Scorpion planner \cite{Seipp17Scorpion}, saving the graphs of the states on the optimal paths as data, and their optimal heuristic values as targets. One challenge we encountered with the Scorpion planner is the limited number of
training problems it is able to solve.

To address this, we developed a simple method creating additional training problems by considering increasingly larger subsets of the goal of an existing problem. Specifically, our method extracts the sequence in which subgoals were last achieved in optimal solutions of training problems. For existing problems with $n$ subgoals, we create $n-1$ new problems, where the goal for the $k$th new problem consists in the $k$ first subgoals of the sequence, while other elements remain unchanged.

These newly generated problems are then solved by the Scorpion planner as well. A key advantage of this method is that the new problems produced are simpler than the originals, enabling more problems to be solved. Furthermore, the method can be considered a form of offline data augmentation, as it doesn't rely on external sources of additional training data.

\subsection{Network Structure}
Our graph network is implemented using the standard PyTorch Geometric package \cite{fey:lenssen:19}. It comprises multiple RGCN layers with a hidden dimension of 64 \cite{schlichtkrull:etal:18}, followed by global add pooling and a linear layer that produces a one-dimensional output representing heuristic values. For most domains, we use three RGCN layers; however, the spanner, logistics and tyreworld domains require four layers, while the sokoban domain employs seven layers.

This layer adjustment addresses an observed issue: with fewer layers, the GNNs produce identical output vectors, leading to lower performance and substantial pruning errors in certain domains. This raises an important question: should the same number of layers be used across all domains? Due to substantial structural differences in \pog \ across domains, a limited number of layers often fails to effectively aggregate information for specific graph types. For example, in the logistics domain, information about the goal location is critical but can be distant from the package's current location in \pog, necessitating additional layers to effectively capture and aggregate this information. Therefore, the number of layers in GNNs for \pog \ should be tailored to each domain.

\subsection{Validation Details}
To ensure result stability and reduce the bias in the training datasets, we apply early stopping based on a validation set to select the optimal models \cite{bai2022rsa, bai2023sdn}. Specifically, since the training data only comes from optimal paths, a prediction that closely approximates the ground truth does not guarantee that it is the lowest in the heuristic queue. To address this issue, we save all sibling states on the optimal paths to select models. The selection metric is the validation accuracy, where the correct result corresponds to the optimal state with the lowest heuristic value among sibling states. Note: for validation, the generated sub-problems are not included.

\subsection{Reproducibility}

To ensure reproducibility, we have implemented several measures. First, we report both the mean and standard deviation of all experimental results. Second, to enable reliable replication, we conduct experiments in PyTorch's deterministic mode throughout model training, ensuring that results are consistent when using the same hardware and software versions. Finally, comprehensive details regarding the hardware, software, container image, and datasets used in this research will be made publicly available in our code repository, which will be released upon publication.

\end{document}